\PassOptionsToPackage{sort&compress}{natbib}   
\documentclass{article}
\PassOptionsToPackage{table}{xcolor}
\usepackage{iclr2027_conference,times}
\setcitestyle{numbers,square}

\usepackage{hyperref}
\usepackage{url}

\usepackage[utf8]{inputenc} 
\usepackage[T1]{fontenc}    
\usepackage{booktabs}       
\usepackage{amsfonts}       
\usepackage{nicefrac}       
\usepackage{microtype}      
\usepackage[table]{xcolor}  
\usepackage{wrapfig}
\usepackage{colortbl}
\definecolor{tablegray}{gray}{0.95}

\usepackage{amsmath}
\usepackage{dsfont, bbm, bbold}
\usepackage{amssymb}
\usepackage{algorithm}
\usepackage{algorithmic}
\usepackage{setspace}
\usepackage{subfig}
\usepackage{adjustbox}
\usepackage{enumitem}
\usepackage{graphicx}

\usepackage{tikz}
\usetikzlibrary{arrows,automata,positioning}
\usetikzlibrary{shapes.multipart}
\usetikzlibrary{decorations.markings}
\usetikzlibrary{decorations.pathreplacing}

\newcommand{\cA}{\mathcal{A}}

\newcommand{\cL}{\mathcal{L}}

\newcommand{\cP}{\mathcal{P}}
\newcommand{\cQ}{\mathcal{Q}}
\newcommand{\cR}{\mathcal{R}}
\newcommand{\cS}{\mathcal{S}}
\newcommand{\cT}{\mathcal{T}}

\newcommand{\diag}{\text{diag}}

\newtheorem{theorem}{Theorem}[section]

\newcommand{\mbf}[1]{{\mathbf{#1}}}

\usepackage{xparse}
\NewDocumentCommand{\expectation}{O{\big} m m}{%
  \mathbb{E}_{#2}{#1[ #3 #1]}%
}

\newcommand{\Zeta}{Z}

\renewcommand{\mathcal}{\mathscr}

\usepackage[mathscr]{eucal}

\title{The Terminal Representation in Reinforcement Learning}

\author{Amir Esterhuysen \& Anders Jonsson \\
Dept. Information and Communication Technologies\\
Universitat Pompeu Fabra\\
Barcelona, Spain \\
\texttt{\{amiryanirmichael.esterhuysen, anders.jonsson\}@upf.edu}
}

\iclrfinalcopy 
\begin{document}

\maketitle
\lhead{Preprint}

\begin{abstract}
Representation learning is a powerful tool for spatio-temporal abstraction within reinforcement learning (RL). Two well established approaches are through the successor representation (SR) and the default representation (DR). The SR encodes states by the future trajectories they induce, capturing information flow decoupled from reward. The DR builds on this by weighting trajectories with reward, integrating credit-assignment structure into the representation. Both representations, as well as their eigenvectors, have been used to support a range of downstream tasks---including option discovery, reward shaping, transfer learning, and exploration. We introduce a structurally distinct formulation: the \textit{terminal representation} (TR). The TR encodes reward-weighted trajectories similarly to the DR, but can be learned as a lower-dimensionality object, and can be used directly for the mentioned applications without eigenvector computations. Eigenvector methods also rely on reversible transition dynamics, which the TR can bypass. In this work we develop the theoretical foundations of the TR: its derivation, convergence of two learning algorithms, its use for zero-shot compositionality, and equivalences between alternative reward formulations. We further show the TR is embedded in the top DR eigenvector, allowing it to capture the same underlying knowledge without eigendecomposition. Additionally, we provide empirical evidence of the TR as a viable alternative to existing representations in subsidiary applications, while requiring less computational overhead to learn, store, and use.
\end{abstract}

\section{Introduction}
Learning features for mapping problems onto more tractable forms is a priority across machine learning \citep{bengio2013representation}. In Reinforcement Learning (RL), this is termed \textit{representation learning}, and typically involves transformation of the state space.

Previous work analyses two related state representations: the \textit{successor representation} (SR) \citep{dayan1993improving} and the \textit{default representation} (DR) \citep{piray2021linear}. In the SR, states are represented via the future visitation trajectories they induce, capturing underlying transition dynamics of the environment. The DR extends this idea to reward-weighted trajectories, and thus exploits knowledge of the credit-assignment structure.

We propose a novel approach, the \textit{terminal representation} (TR), related to but distinct from the DR.
The TR represents reward-weighted visitations of terminal states and has the following advantages:

\begin{itemize}
    \item \textbf{Compactness} - The TR is lower dimensionality than the SR and DR, aiding more efficient learning and storage.
    \item \textbf{No Additional Computation} - The DR/SR rely on eigendecomposition for downstream tasks like reward shaping and option discovery, while the TR can be used directly as is.
    \item \textbf{Handling Asymmetry} - The above eigenvector approach relies on reversible transition dynamics, while the TR does not make such assumptions.
    \item \textbf{Terminal Knowledge} - The TR directly encodes information about the reachability of terminal states in the environment, weighted by the non-terminal rewards along this path---facilitating zero-shot transfer to any choice of terminal rewards.
\end{itemize}
In this work, we formally derive the TR and analyse its structure to arrive at these benefits. We prove the convergence of two learning algorithms under appropriate conditions, illustrate how the TR can be used to achieve zero-shot compositionality, demonstrate structural overlap between the TR and the top eigenvector of the DR, and handle equivalence between varying reward formulations.

These theoretical properties are supplemented by empirical results, where the TR in its base form is competitive with the SR and DR across four areas: 1) \textit{option discovery}, 2) \textit{reward shaping}, 3) \textit{transfer learning}, and 4) \textit{count based exploration}. Through this empirical analysis we maintain an important thread: the TR achieves this comparable performance while requiring less computationally intensive methods.

\section{Background}
\vspace{-10pt}
\textbf{Markov decision process:} A \textit{terminating} Markov decision process (MDP)
can be expressed as $\mathcal{M} =(\mathcal{S},\cT, \mathcal{A},\mathcal{P},\mathcal{R})$, where $\mathcal{S}$ and $\mathcal{T}$ are respectively finite, \textit{non-empty} sets of \textit{non-terminal} and \textit{terminal} states. The full set of states is denoted by $\mathcal{S}^+ = \mathcal{S} \cup \mathcal{T}$. We use $S=|\mathcal{S}|$ and $T=|\mathcal{T}|$ to denote the number of non-terminal and terminal states, respectively, and $S^+=|\mathcal{S}^+|$ to denote the total number of states. $\mathcal{A}$ is the finite action set. $\mathcal{P}: \mathcal{S} \times \mathcal{A} \to \Delta(\mathcal{S^+})$ represents the transition dynamics, while $\mathcal{R}:\mathcal{S^+} \to \mathbb{R}$ is a \textit{reward function}. We assume that $\cR$ is state-dependent and that $\cR(s)<0$ for all $s\in\cS$.

At each time-step $t$, the learning agent receives a state $s_t \in \cS$ and executes an action $a_t \in \cA$. The environment sends back a reward $r_t = \mathcal{R}(s_t)$, and moves the agent to a new state $s_{t+1} \sim \mathcal{P}(\cdot | s_t, a_t)$.

The agent aims to compute a stochastic \textit{policy} $\pi:\mathcal{S}\rightarrow\Delta(\mathcal{A})$ that maximizes its expected future reward. Each $\pi$ induces a \textit{transition kernel} $\cP^{\pi} : \cS\to\Delta(\cS^+)$ which consolidates all transition effects such that $\cP^{\pi}(s'\mid s) = \sum_{a\in\cA}\pi(a\mid s) \cdot \cP(s'\mid s, a)$. We adopt the convention that $\cP^{\pi}(s'\mid \tau)=0$ for all $s',\tau\in\cS^+\times\cT$, i.e.~terminal states do not induce any transition probability. Each policy is assumed to be \textit{proper}, i.e.~guaranteed to reach a terminal state. That is, each $\pi$ induces trajectories of the form $(s_1,a_1,\dots a_{N-1},s_N)$ such that $s_N\in\cT$, where $N$ is a discrete random variable representing episode length. In communicating MDPs, this is achieved by assigning non-zero probability to each action.

Another central quantity is the \textit{value function} of $\pi$, which we define in each non-terminal state $s \in \mathcal{S}$ as $v^{\pi}(s) = \mathbb{E}\big[\sum_{t=1}^{N-1} \mathcal{R}(s_t) + \mathcal{R}(s_N)\mid s_1 = s\big]$, where the expectation is over the stochastic choice of next state $s_{t+1} \sim \cP^{\pi}(\cdot\mid s_t)$ at each time $t$, and the time $N$ it takes for the episode to end. Since $\mathcal{R}(s) < 0$ for each $s \in \mathcal{S}$, the value $v^{\pi}(s)$ has a well-defined upper bound. We extend the value function to each terminal state $\tau \in \cT$ by defining $v(\tau) \equiv \cR(\tau)$. The optimal policy is $\pi^* = \arg\max_\pi v^\pi$, and its corresponding optimal value function satisfies the Bellman optimality equations $v^*(s) = \mathcal{R}(s) + \max_{a \in \mathcal{A}} \sum_{s' \in \mathcal{S}^+} \mathcal{P}(s' | s, a) v^*(s')$ for all $s\in\cS$. Note the non-linear $\max$ operation present in this equation, which limits tractability of solution methods.  


\textbf{Linearly-solvable Markov decision process:} We consider a related MDP formulation, with properties that aid tractability. A \textit{linearly-solvable Markov decision process} (LMDP)~\citep{TodorovNIPS2007} can be defined as $\mathcal{L} = (\mathcal{S}, \mathcal{T}, \cA, \mathcal{P}, \mathcal{R}, \mu)$, i.e.~an MDP augmented with a \textit{default policy} $\mu: \cS \to \Delta(\cA)$, representing the uncontrolled way of acting in this environment. All other definitions are inherited from the MDP setting. We additionally assume that any transition kernel $\cQ : \cS\to\Delta(\cS^+)$ defined for each $s,s'\in\cS\times\cS^+$ as $\cQ(s'\mid s)=w(s')\cP^\mu(s'\mid s)/Z$ is \textit{realizable}, i.e.~there exists a policy $\pi$ such that $\cP^\pi=\cQ$. Here, $w(s')>0$ is an arbitrary weight of state $s'$ and $Z$ is a normalization term. For example, this condition is satisfied for deterministic MDPs.

The immediate reward of an LMDP is $\mathcal{R}(s_t,\pi)= \mathcal{R}(s_t) - \lambda\cdot\mathrm{KL}(\cP^{\pi}(\cdot|s_t)\Vert\, \cP^{\mu}(\cdot|s_t))$. Hence, the agent can choose $\pi(\cdot\mid s_t)$ freely, but is penalized by the KL-divergence between the induced kernel $\cP^{\pi}(\cdot\mid s_t)$ and the default kernel $\cP^{\mu}(\cdot\mid s_t)$. The strength of this penalty is regulated by the temperature parameter $\lambda$.

Expressing our optimal value function as $v$, the Bellman optimality equations can be written as
\begin{align*}
& \frac 1 \lambda v(s) = \frac 1 \lambda \max_\pi \left[ \cR(s,\pi) + \mathbb{E}_{s'\sim\cP^{\pi}(\cdot|s)} v(s') \right] \\
 &= \frac{1}{\lambda} \cR(s) + \max_\pi \mathbb{E}_{s'\sim\cP^{\pi}(\cdot|s)} \left[ \frac 1 \lambda v(s') - \log \frac {\cP^{\pi}(s'|s)} {\cP^{\mu}(s'|s)} \right] \;\; \forall s.
\end{align*}
We introduce the notation $z(s)=e^{v(s)/\lambda}$ for each $s\in\cS^+$, and will often refer to $z(s)$ simply as the value of $s$. In this format, we can resolve the maximization operation inside the Bellman equations analytically, to achieve
\begin{equation}\label{eq:z}
z(s) = e^{\cR(s)/\lambda} \sum_{s'}\cP^{\mu}(s'|s)z(s').
\end{equation}
The optimal control is given by a transition kernel $\cP^*$ defined for each $s,s'\in\cS\times\cS^+$ as
\begin{align*}
\cP^*(s'\mid s) = \frac {\cP^{\mu}(s'\mid s)z(s')} {\sum_{s''} \cP^{\mu}(s''\mid s)z(s'')}.
\end{align*}
By assumption, there exists a policy $\pi$ such that $\cP^\pi=\cP^*$. In practice, the learner directly maintains an estimate of $\cP^*$ from which it samples the next state, and hence never needs to compute $\pi$.

We define the following vectors and matrices. Let $R$ be an $S^+\times S^+$ diagonal matrix with elements $R(s,s)=e^{\cR(s)/\lambda}$, and let $P$ be an $S^+\times S^+$ matrix with elements $P(s,s')=\cP^{\mu}(s'\mid s)$. $P_S$ is the $S\times S$ restriction of $P$ that considers only transitions between non-terminal states, while $P_T$ is the $S\times T$ restriction that considers transitions from non-terminal to terminal states. $R_S\in\mathbb{R}^{S\times S}$ is the non-terminal restriction of $R$, and $y$ is a $T$-dimensional vector with elements $e^{\cR(\tau)/\lambda}$. We let $z$ be an $S$-dimensional vector with elements $e^{v(s)/\lambda}$.

\textbf{Successor representation:} The successor representation (SR) \citep{dayan1993improving} encodes a state as the expected number of visits to its successor states, under a given policy $\pi$. This is built on the intuition that similar states lead to similar outcomes. The undiscounted SR, $\mathbf{\Psi}^\pi \in \mathbb{R}^{S^+ \times S^+}$, is defined as
\begin{align}
\mbf{\Psi}^\pi
&= \sum_{t=0}^\infty (\mbf{P}^\pi)^t = \left(\mbf{I} - \mbf{P}^{\pi}\right)^{-1}.
\end{align}
This quantity is well defined in the undiscounted setting under the previous assumptions of our framework: policies are proper and terminal states do not induce transition probability. The SR can be learned online using temporal difference (TD) methods \citep{dayan1993improving}. Once obtained, the SR for a policy $\pi$ can be used to express the value function as $v^{\pi}(s)=\sum_{s'\in\cS^+} \mbf \Psi^\pi (s,s')\cR(s')$.

\textbf{Default representation:} The SR is defined without use of the reward function, isolating the policy-dependent transition dynamics. We move now to a related concept, the default representation (DR), defined in the LMDP setting \citep{piray2021linear,tse2025reward}. The DR is 'reward-aware' in that it encodes states by the reward-weighted trajectories they induce.

Denoted $\mathbf{\Zeta}\in\mathbb{R}^{S^+\times S^+}$, the DR can be computed in closed form by
\begin{align}\label{eq:DR_definition}
    \mbf \Zeta &= (R^{-1} - P )^{-1} = (I-RP)^{-1}R.
\end{align}
Given knowledge of transition dynamics, the DR can be used to compute the optimal value function for any configuration of terminal rewards \citep{piray2021linear, tse2025reward}. We have $z = \mbf \Zeta_{S} P_{T}y$ where $\mbf\Zeta_{S}=( R_S^{-1} - P_S )^{-1}$ is the $S\times S$ restriction of $\mbf\Zeta$. The DR can be learned either using a dynamic programming (DP) formulation (in the case of known dynamics) or through TD learning.

\section{Related work}
The terminal representation is related to the SR \citep{dayan1993improving} and the DR \citep{piray2021linear,tse2025reward}, both mechanically and semantically. We explore these connections throughout, but essentially the DR extends the SR's transition encodings with reward knowledge and the TR extends the DR with terminal state knowledge. The TR can also be interpreted as a special case of \textit{default features} \citep{tse2025reward}, where the choice of features indexes over terminal states. The use of representations that carry underlying information flow in MDPs is influenced by \textit{proto value function} (PVF) theory \citep{mahadevan2007proto}, which led subsequently to the idea of \textit{eigenoptions} which highlight this information flow in the context of temporally extended behaviour \citep{machado2017laplacian, machado2018eigenoption, machado2023temporal}. \citet{machado2023temporal} in particular frame this as a cycle in which the representation informs option discovery and the resulting options refine the representation.

The thread of value function composition in LMDPs that we study can be linked to the work of \citet{todorov2009compositionality}, but our direct application of the TR in this vein is related to \citet{infante2022globally} who apply compositionality in a hierarchical setting. This hierarchical use of LMDPs has a longer history: \citet{jonsson2016hierarchical} formulate each task in a hierarchy as its own LMDP, while \citet{saxe2017hierarchy} exploit concurrent compositionality to blend many subtasks in parallel rather than serially. \citet{infante2024hierarchical} extend the hierarchical LMDP setting to average reward, partitioning the state space into subtasks and using compositionality to represent the high-level value function exactly. 

Compositionality has overlap with the idea of transfer learning, and our investigation here can be compared to the \textit{successor features} approach \citep{barreto2017successor}. Subsequent work extends this line in relevant directions: \citet{barreto2019option} combine skills in the space of cumulants to synthesise new options without learning, \citet{alegre2022optimistic} characterise when combined policies are optimal rather than merely improved, and \citet{kuric2024planning} learn a policy basis over subproblems that can be recombined to solve tasks specified by an automaton. These share our interest in a reusable basis that is invariant up to task specification, though they operate over reward features rather than terminal structure.


\section{Terminal representation for LMDPs}
This section forms the core of our work. We derive the terminal representation within the LMDP framework, and discuss how it can be learned recursively---through both a dynamic programming approach under the assumption of known dynamics, and in purely sample-based fashion.

We exploit a feature of the LMDP setting: our linear Bellman optimality equation can be decomposed into two terms, corresponding respectively to the non-terminal and terminal states:
\[
  z(s)
  = e^{\cR(s)/\lambda}\sum_{s'\in \cS}\cP^{\mu}(s'\mid s)z(s')
    + e^{\cR(s)/\lambda}\sum_{\tau\in \cT}\cP^{\mu}(\tau\mid s)e^{\cR(\tau)/\lambda},
  \quad \forall s\in \cS.
\]

This leads to matrix equations in the form
\[
  z = R_SP_Sz + R_SP_Ty
  \iff
  R_S^{-1}z = P_Sz + P_Ty
  \iff
  z = (R_S^{-1}-P_S)^{-1}P_Ty.
\]

Jointly considering the influence of environment transitions and rewards, we let $D_S=R_SP_S$ be the $S\times S$ non-terminal matrix, and $D_T=R_SP_T$ be its $S\times T$ equivalent.

Then, the matrix $M = (R_S^{-1}-P_S)^{-1} P_T = (I_S-R_SP_S)^{-1} R_SP_T = (I_S-D_S)^{-1} D_T$

is the terminal representation (TR) of the LMDP $\cL$. Given any configuration of terminal rewards, we can use the TR to zero-shot recover the optimal $z$-values across all non-terminal states: $z = My$.

\textbf{Learning the terminal representation:} We can reformat the expression for $M$ as a recursive dynamic programming update rule:

$M = (I_S-D_S)^{-1}D_T = \left( \sum_{i=0}^{\infty} (D_S)^i \right) D_T = D_T + D_S M.$

\begin{theorem}\label{dp-conv}
    Let $ M_0 = D_T$. The update rule
    \begin{align}
        M_{k + 1} = D_T + D_S M_k
    \end{align}
    converges to the TR: $\lim_{k \to \infty} M_{k} = M$.
\end{theorem}

\emph{Proof.} Found in Section \ref{dp-proof} of the appendix. $\blacksquare$

To handle transitions into terminal states in our vector updates, we define an extended mapping $M^+: \mathcal{S}^+\times\cT \to \mathbb{R}$ such that

$M^+(s,\tau_i) =
\begin{cases}
M(s,\tau_i) & \text{if } s \in \mathcal{S}, \\
\mathbb{1}_{s=\tau_i} & \text{if } s \in \mathcal{T}.
\end{cases}$

For row-wise operations, we adopt the convention that $M^+(s)=M^+(s,\cdot)\in\mathbb{R}^T$.

This leads to an online update rule for learning rows of $M$ from samples of the form $(s_t, a_t,r_t, s_{t+1})$, drawn under the default policy distribution $\cP^{\mu}$:
\begin{equation}\label{td-update}
    \widehat M(s_t) \leftarrow (1-\alpha_t) \widehat M(s_t) + \alpha_te^{\frac{r_t}{\lambda}} \widehat M^+(s_{t+1})
\end{equation}

$\widehat M$ denotes the current estimate of $M$ and $\widehat M^+$ its extended domain equivalent.

\begin{theorem}\label{td-conv}
    Assume the learning rate $\alpha_t$ is such that $\sum_t^{\infty}\alpha_t =\infty$ and $\sum_t^{\infty}\alpha_t^2 <\infty$. The sample-based update rule in \eqref{td-update}
    will converge to $M$.
\end{theorem}
\emph{Proof:} Found in Section \ref{td-proof} of the appendix. $\blacksquare$

\section{Structure of the terminal representation}\label{translation}

In this section, we present a more detailed analysis of the structural properties exhibited by the terminal representation, particularly their contrasts with those of the default representation, and the consequent benefits. 

\textbf{Dimensionality:} A straightforward advantage of the TR is compactness, allowing in practice for faster learning and more efficient storage compared to the other discussed representations.

Consider the TR expressed in the form $M = (R_S^{-1}-P_S)^{-1} P_T = \mbf Z_S P_T$, where $\mbf Z_S$ is the non-terminal restriction of the DR. It is simple to verify that $M\in\mathbb{R}^{S\times T}$. In our setting---with $\cS,\cT\ne\emptyset$---it is necessary that $S,T< S^+$. Thus $M$ is guaranteed to be smaller than the $S^+\times S^+$ DR.

While not a requirement, often the number of terminal states $T$ is much smaller than $S$.
Thus we can also say that the TR tends to be lower-dimensional than the $S\times S$ non-terminal restriction of the DR that is used in applications like zero-shot value function recovery.

\textbf{Eigenvectors:} The DR (as well as the SR) has been used in applications like reward shaping and option discovery. This is done through manipulation of the top eigenvector. Typically, performing a full eigendecomposition is $O((S^+)^3)$ \citep{golub2013matrix}, while methods like power iteration that return only the top eigenvector can cost $O((S^+)^2)$ per iteration \citep{saad2011numerical}. For reward shaping, the DR and its top eigenvector are computed offline under the assumption of known dynamics. The eigenvector is then queried at each Q-learning iteration to generate an augmented learning signal. For option discovery, the DR is learnt in parallel with the value function. At each iteration we 1) update the representation, 2) compute its top eigenvector, 3) use this to generate options which are added to the agent's action set, and 4) repeat the process with this newly extended action space. If $k$ iterations of the algorithm run, this leads to an eigenvector cost of at least $O(k\cdot(S^+)^2)$.

The TR can be used for the same applications. However, we gain a powerful advantage---the columns of the TR are directly extracted and used equivalently without additional computation. Alternative methods like \textit{slow feature analysis} \citep{seabrook2024relation} learn eigenvectors directly through an online algorithm, but this in turn sacrifices access to the full predictive matrix that we require for applications like zero-shot value recovery. Learning the TR online provides access to a rich predictive matrix, while also capturing the utility of eigenvector extractions---without needing to compute eigenvectors. We formalise this relationship in the following theorem:

\begin{theorem}\label{dr-eigenvec}
Let $\mathbf{\Zeta}$ be the DR over $\mathcal{S}^+$ and $
\mbf \Zeta_S$ its restriction to non-terminal states. Let $\tau^* = \mathrm{argmax}_{\tau\in\mathcal{T}}\mathcal{R}(\tau)$, $\bf e_{\tau^*}$ be a one-hot vector of length $T$ that encodes $\tau^*$, and $M_{\cdot, \tau^*}$ be the corresponding column vector of the TR $M$. Suppose $e^{\cR(\tau^*)/\lambda} > \rho(\mbf \Zeta_S)$, where $\rho(\cdot)$ denotes spectral radius, which holds in particular whenever
$e^{\cR(\tau^*)/\lambda} > \frac{\gamma}{1-\gamma}$ for
$\gamma := \max_{s\in\cS} e^{\cR(s)/\lambda} < 1$. Then the top eigenvector of $\mathbf{\Zeta}$ is
$$\mbf v^* = \begin{bmatrix} \sum_{k=0}^{\infty} e^{-k\frac{\mathcal{R}(\tau^*)}{\lambda}} \mathbf{\Zeta}_{S}^k M_{\cdot, \tau^*} \\ \bf e_{\tau^*} \end{bmatrix}$$
with corresponding eigenvalue $\beta^* = e^{\mathcal{R}(\tau^*)/\lambda}$.
\end{theorem}

\emph{Proof:} Found in Section \ref{eigenvec-proof} of the appendix. $\blacksquare$

So, the TR appears as a central component inside the definition of the DR's top eigenvector. Using the TR columns directly captures the same underlying information, while avoiding any eigendecomposition. This reduces complexity per vector query from $O((S^+)^2)$ or worse to an $O(1)$ operation.




Further analysis (Appendix~\ref{eigenvec-proof}) shows that the non-terminal component $\mbf v^*_S$ of $\mbf v^*$ can be interpreted as a discounted reward-aware transition term over trajectories of varying lengths, weighted by transitions to the most desirable terminal state.

\textbf{Continuous states:} We introduce our setting with discrete state space $\cS$, however this is not inherently required for the construction of $M$. For continuous $\cS$ and finite $\cT$, $D_S$ and $D_T$ become operators on bounded functions, $(D_Sz)(s)=e^{\cR(s)/\lambda}\int_{\cS}z(s')\,\cP^\mu(ds'\mid s)$ and $(D_Ty)(s)=e^{\cR(s)/\lambda}\sum_{\tau\in\cT}\cP^\mu(\tau\mid s)\,y(\tau)$. Then $\|D_S\|_\infty<1$, the series $M=\sum_{i\ge0}D_S^iD_T$ converges, $M(s,\tau)$ is defined pointwise, and $z=My$ still recovers the optimal values (see Appendix Section \ref{cont-proof}). 

Each row-analogue $M(s,\cdot)$ remains a vector in $\mathbb{R}^T$, making the TR a natural target for parameterised function approximation of the form $f_\theta:\cS\to\mathbb{R}^T$. Additionally, the TR recursion $M=D_T+D_SM$ has the sample target $e^{r_t/\lambda}M^+(s_{t+1})$, which directly estimates reward-weighted terminal occupancy and never has to compute an estimate of the terminal transition kernel $\cP^{\mu}(\tau\mid s)$ from $s\in\cS$ to $\tau\in\cT$---which in practice may be difficult or impossible to do over large or continuous spaces.

\textbf{Symmetry:} The underlying information structure within MDPs has been studied using the graph Laplacian $L = D - W$, where $W$ is a symmetric weight matrix encoding state adjacency and $D$ is the corresponding degree matrix \citep{mahadevan2007proto}. This includes the notion of \textit{proto-value functions} (PVF), basis functions that are derived from the eigenvectors of $L$. Eigendecomposition of the DR (or SR) exploits equivalence to the PVF eigenvectors through an identity of the form $P^\pi = D^{-1}W$, which
requires that the distribution induced by any policy $\pi$ used to compute the SR/DR is reversible: there must exist another distribution $\nu$ such that $\nu(s)\,\mathcal{P}^\pi(s'\mid s) = \nu(s')\, \mathcal{P}^\pi(s\mid s')$ for all $s, s' \in \mathcal{S}^+$.

Reversibility ensures $P^\pi$ is similar to a symmetric matrix, guaranteeing real eigenvalues for the DR. Without it, the DR's eigendecomposition may yield complex eigenvalues and non-orthogonal eigenvectors, undermining its use as a structural representation of the environment. This is a strong condition. Policies are typically directional, biasing the agent toward certain states and
away from others. Such policies are incompatible with the balanced flow of behaviour that reversibility demands.

The TR is constructed in a way that avoids this issue. It is of shape $S \times T$ and maps non-terminal to terminal states, making it inherently asymmetric and avoiding the need for reversibility. Its columns encode the same structural information that the DR's top eigenvector captures under reversibility (Theorem~\ref{dr-eigenvec}), but $M$ is well-defined and directly usable under arbitrary non-reversible dynamics.

\textbf{Hierarchy:} We believe the TR will become especially useful in hierarchical RL, where an overarching decision problem is decomposed into a set of subproblems. Whether handcrafted or learned, these subproblems are naturally defined using terminal states whose values may not be known at the outset of learning. Further, two or more subproblems may share both dynamics and terminal states yet assign different terminal rewards.
In this setting, estimating the TR allows zero-shot learning of subproblem value functions given a set of terminal rewards. The SR and DR do not exhibit this property, and subproblem value functions may only be encoded after additional work to estimate the transition kernel $\cP^{\mu}(\tau\mid s)$ from $s\in\cS$ to $\tau\in\cT$.

\textbf{Limitations:} The TR is defined such that terminal states and their reward/transition dynamics are segmented from non-terminal states. This relies on the assumption of 1) an environment with clearly demarcated terminal states and 2) some structural knowledge of these states. The DR, defined over the full state space, does not require these conditions to be satisfied. This is a clear limitation of the TR, and restricts the environments where it can be applied. However, it is still important to acknowledge certain mitigating factors:
\begin{itemize}[itemsep=0pt, parsep=0pt]
    \item The presence of terminal or 'goal' states is a common feature in RL environments, and does not constitute a fringe setting.
    \item Given such states, encoding direct knowledge of their relative desirability and reachability helps facilitate efficient learning.
    \item Thus far we have been using the language of terminal states, but it is more accurate to say the TR requires a set of demarcated states, which do not necessarily have to induce termination.
\end{itemize}
The third point is worth emphasising. The TR theory can function with arbitrarily selected 'relevant' states---high or low reward regions, bottleneck states that lead from one segment of the environment to another, boundary states. Pivoting to this view does not reproduce the flexibility offered by the DR, but we can relax the need for traditionally defined goals. Future work should formalise this extension, along with methods that automatically discover important states during learning.

\section{Additional Properties: Compositionality and Alternative Reward Formulations}

\textbf{Reward:} The LMDP framework was introduced with state-dependent rewards. It is often the case that MDPs and LMDPs are defined with rewards that depend either on the transition $(s,s')$ or the state-action pairing $(s,a)$, which may require a modified analysis. We deal with both these cases by showing 1) that every LMDP with transition-dependent reward can be mapped to one with state-dependent reward, and that 2) the TR can be extended to the case of state-action rewards.

\begin{theorem}\label{equiv}
    Consider an LMDP of the form $\widetilde{\cL}=\left( \cS,\cT,\cA,\cP,\widetilde{\cR},\widetilde{\mu}\right)$, with \textit{transition} dependent reward $\widetilde\cR:\cS\times\cS^+\to\mathbb{R}$. We can always transform $\widetilde \cL$ to an equivalent LMDP
\[
  \cL=\left( \cS,\cT,\cA,\cP,\cR,\mu\right)
\]
with state-dependent rewards $\cR:\cS^+\to\mathbb{R}$, such that $\cL$ and $\widetilde{\cL}$ have the same optimal value function.
\end{theorem}

\emph{Proof:}
Found in section \ref{equiv-proof} of the appendix. $\blacksquare$

Owing to the equivalence of $\cL$ and $\widetilde{\cL}$, it is sufficient to consider LMDPs with state-dependent rewards.

\begin{theorem}\label{sa-tr}
    Consider an LMDP of the form $\bar{\cL}=(\cS,\cT,\cA,\cP,\bar{\cR},\mu)$, where the reward is given by $\bar \cR :\cS^+\times\cA\to\mathbb{R}$. Let $q$ be the optimal action-value function satisfying Bellman equations of the form $q=\bar{\mathcal R}+\lambda\log\mathbb{E}_{s'}\big[\sum_{a'}\mu(a'|s')e^{q(s',a')/\lambda}\big]$. Now we define the state-action exponentiated values as $\bar z(s,a)=e^\frac{q(s,a)}{\lambda}$ for $s\in\cS$ and $\bar z(\tau,a)=e^\frac{\bar \cR(\tau,a)}{\lambda}$ for $\tau\in\cT$. Let $\bar z \in \mathbb{R}^{S\times A}$ be the vector of exponentiated non-terminal action-values and $\bar y \in \mathbb{R}^{T\times A}$ its terminal equivalent. Then, we can compute the state-action TR $\bar M$ under this formulation and use it to recover action-values through the relation $\bar z = \bar M\bar y$.
\end{theorem}

\emph{Proof:} Found in Section \ref{sa-proof} of the appendix. $\blacksquare$

\textbf{Composition:} An attractive attribute of LMDPs is the idea of \textit{compositionality} \citep{todorov2009compositionality}. Given value functions for a set of base LMDPs, we can automatically recover a value function for the overarching LMDP that represents their combination. We formalise this idea as follows:

Consider a set of LMDPs $\{\cL_1,\dots,\cL_T\}$, where each $\mathcal{L}_i$ is given by $(\mathcal{S}, \mathcal{T}, \cA, \mathcal{P}, \mathcal{R}_i, \mu)$. These LMDPs differ only in the reward at terminal states, i.e.~$\cR_i(s)=\cR_j(s)$ for all $s\in\cS$.

Now let there be a single LMDP $\cL=(\mathcal{S}, \mathcal{T}, \cA, \mathcal{P}, \mathcal{R}, \mu)$, with an associated set of weights $\{w_1,\dots,w_T\}$. Assume $z(\tau)=\sum_{i=1}^Tw_iz_i(\tau)$ for all $\tau\in\cT$. The Bellman equation is linear over $z$ for non-terminal states, and so $z(s)=\sum_{i=1}^Tw_iz_i(s)$ for all $s\in\cS$.

We will show that the TR can also be used to perform compositionality in this fashion, bypassing the need to learn the base LMDP value functions.

\begin{theorem}\label{tr-comp}
Let $z_i$ be the value of the base LMDP $\cL_i$ with terminal values $z_i(\tau)=\mathbb{1}_{\tau=\tau_i}$ (i.e.~$\cR_i(\tau_i)=0$, $\cR_i(\tau)=-\infty$ otherwise). Then $M^+(\cdot,\tau_i)=z_i$ for $i=1,\dots,T$, and hence the TR can be used to perform compositionality via the relation $z(s)=\sum_{i=1}^T z(\tau_i)M^+(s,\tau_i)$.
\end{theorem}

\emph{Proof:} Found in Section \ref{comp-proof} of the appendix. $\blacksquare$

This result allows us to perform compositionality using only the TR, which encodes knowledge of all constituent LMDPs and eliminates the need to obtain them independently.

\section{Experiments}
We mirror the experimental setup used by \citet{tse2025reward}. Their work is focused on using the DR as an alternative to the SR, specifically analysing the impact of reward awareness. We compare the TR to both these approaches. Our narrative is not fixated on the TR as fundamentally superior across these settings, and in fact there is semantic similarity between the TR and the DR. Rather, we emphasise the competitive performance of the TR, while keeping in mind the structural advantages that have been discussed throughout.

\textbf{Option discovery:} Options are temporally extended actions made up of multiple 'primitive actions' \citep{sutton1999between}, defined by the tuple $o=(I_o,\pi_o,\beta_o)$. $I_o\subset\cS$ is the set of states where the option may be initiated, $\pi_o$ is the option policy followed by the agent, and $\beta_o:\cS\to[0,1]$ controls option termination.

The SR~\citep{machado2023temporal} and DR~\citep{tse2025reward} have been used to discover options for exploration. This is done through an iterative online algorithm: as the representation is learned, its top eigenvector generates options that are added to the action set. The SR variant is titled \textit{covering eigenoptions} (CEO), while the DR equivalent is \textit{reward-aware covering eigenoptions} (RACE). The TR can also be used in this cycle, but instead of computing the top eigenvector, TR columns are directly extracted. Four variants are considered: 1) use the column $M_{\cdot, \tau^*}$ of highest-reward terminal state, 2) use all columns, 3) randomly select a column, 4) cycle through the columns. We evaluate these variants against CEO, RACE, and a random walk baseline in two settings: When option discovery aids exploration during learning, and when option discovery is used to collect samples for offline Q-learning. We use the state-action TR formulation to perform Q-learning experiments (Theorem \ref{sa-tr}).

\noindent
\begin{minipage}{\linewidth}
  \centering
  \includegraphics[width=0.49\linewidth]{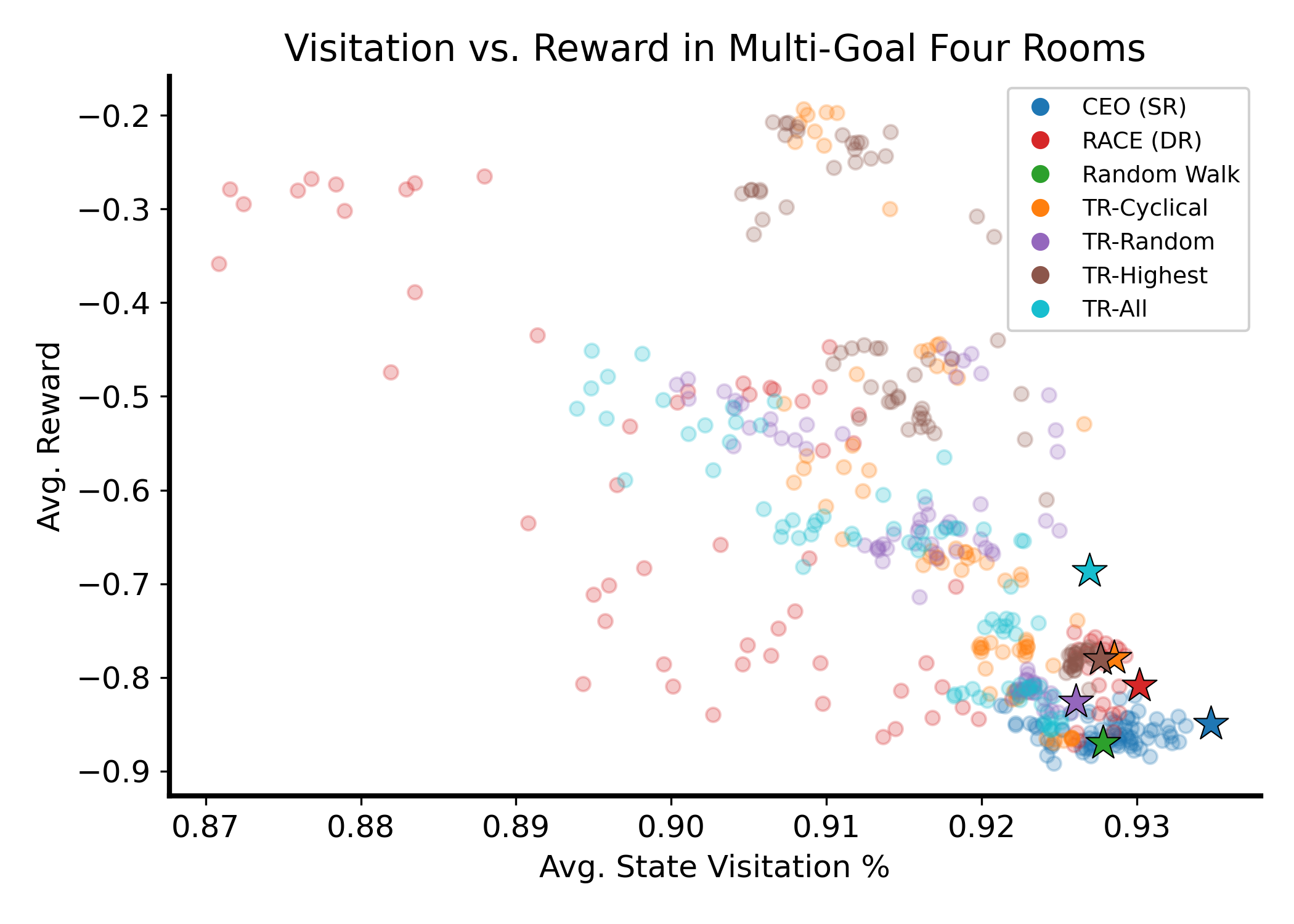}
  \hfill
  \includegraphics[width=0.49\linewidth]{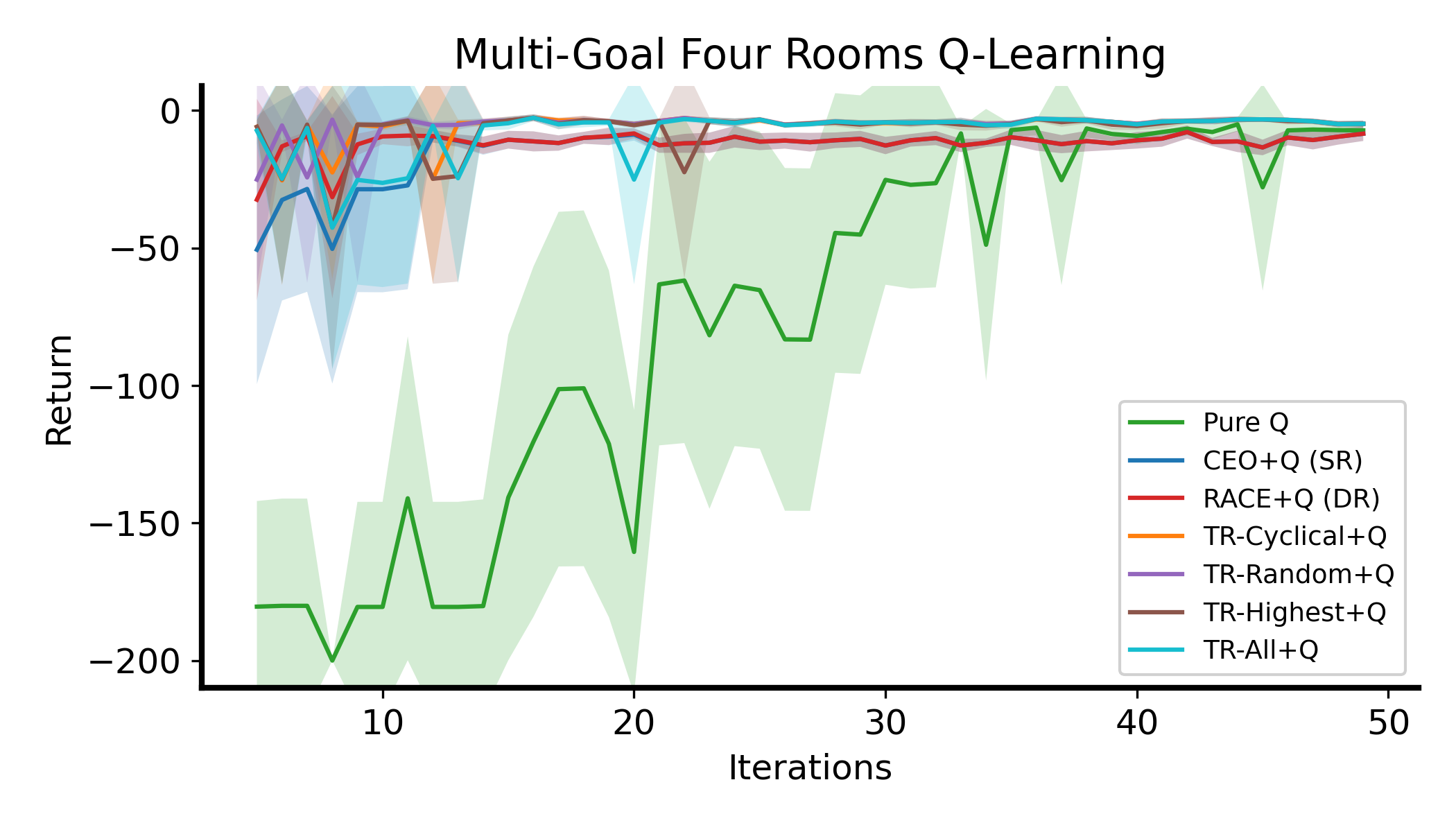}
  \captionof{figure}{
  \textbf{Left:} TR option discovery variants vs. baselines using the DR (RACE), the SR (CEO), and a random walk. Experiments conducted in Four Rooms environment with 4 differently weighted goals. Each dot is one seed and stars mark each method's best seed across the 10 tested.
  \textbf{Right:} Q-learning boosted by different representations: all TR variants, the DR (RACE+Q), the SR (CEO+Q).
  }
\end{minipage}

All TR variants produce competitive performance. Using all column vectors achieves higher average reward than the other representations. Reward-awareness of the DR leads to better returns and worse visitation than the SR \citep{tse2025reward}. The TR stretches this further, by pushing the agent directly to more desirable goals.

\textbf{Reward shaping:}
RL environments are often defined with \textit{sparse rewards}, where the agent only receives instructive feedback at the end of an episode. This can cause inefficient learning. \textit{Reward shaping} combats this, letting the agent augment its reward signal with internal knowledge \citep{ng1999policy}. If $\mbf e$ denotes the top eigenvector of either the SR or DR, $\hat r_t = - \big(\mbf e(s_\text{goal}) - \mbf e(s_{t+1}) \big)^2$ and $\hat r_t = \mbf e(s_{t+1}) - \mbf e(s_t)$ are two choices of shaping reward used in previous work. For any terminal state $\tau\in\cT$, we form the potential-based shaping reward $\hat r_t = M^+(\tau,\tau) - M^+(s_{t+1}, \tau) = 1 - M(s_{t+1}, \tau)$.
\setlength{\belowcaptionskip}{-8pt}
\setlength{\abovecaptionskip}{2pt}
\begin{figure}[H]
  \vspace{-8pt}
  \centering
  \includegraphics[width=0.9\linewidth]{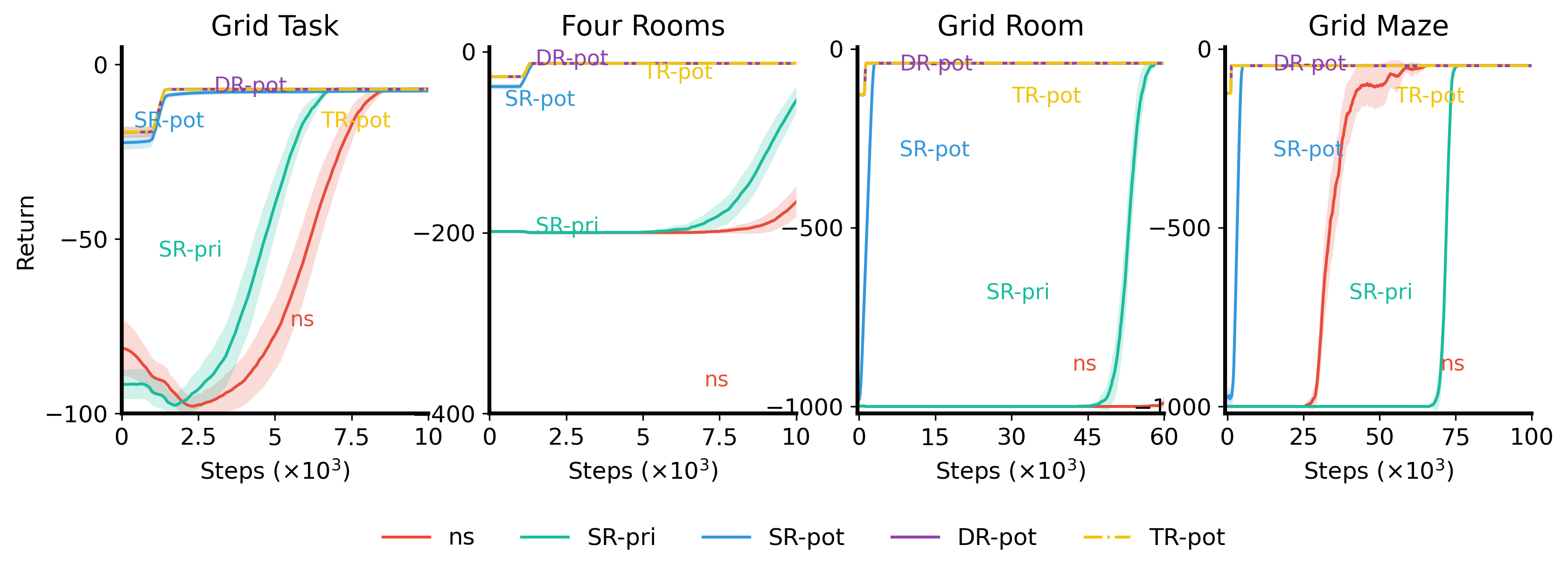}
  \caption{Reward shaping in the indicated environments, averaged over 50 seeds (TR in yellow).}
  \vspace{-10pt}
\end{figure}
Potential-based shaping with the TR achieves identical performance to the DR, successfully navigating through low reward regions---again with the caveat that this is done through direct column extraction. Both reward-aware methods outperform the SR in converging to optimal return.

\textbf{Count based exploration:} The SR captures state visitation counts, and thus its norm can be used as a bonus to promote exploration \citep{machado2020count}. The same principle is applied with the DR norm \citep{tse2025reward}, which has been shown to outperform the SR as an exploration bonus. We evaluate the TR norm under the same conditions---in the RiverSwim and SixArms exploration environments. Comparison is done between unmodified Sarsa, and Sarsa with each of the three norms used as an exploration bonus.

\noindent\begin{minipage}{\linewidth}
\centering
\small
\setlength{\tabcolsep}{6pt}
\captionof{table}{Returns in exploration tasks when using SARSA boosted by each representation as an exploration bonus. Values are $\times 10^3$. 95\% confidence intervals shown in parentheses.}
\label{table:count_based_results}
\rowcolors{2}{white}{tablegray}
\begin{tabular}{@{}lcccc@{}}
\toprule
\textbf{Environment} & \textbf{Sarsa} & \textbf{+SR} & \textbf{+DR} & \textbf{+TR} \\
\midrule
\textsc{RiverSwim} & 23 (0.1) & 98 (25.4) & 2,933 (22.2) & 2,906 (19.3)\\
\textsc{SixArms}   & 276 (9.4) & 678 (72.1) & 3,127 (214.2) & 2,443 (93.6)\\
\bottomrule
\end{tabular}
\end{minipage}

The TR achieves competitive performance with the DR, particularly in RiverSwim where results are near-identical. Both reward-aware representations significantly outperform the SR.

\textbf{Transfer:} The SR, which decouples environment dynamics and reward, has been used for transfer learning in the case where the reward function is altered. This is done through the extension to \textit{successor features} (SF) \citep{barreto2017successor}. SF is deployed alongside \textit{generalised policy improvement} to compute a policy for a new task that combines behaviour from previously learned policies.

\begin{figure}[H]
  \centering
  \includegraphics[width=0.49\linewidth]{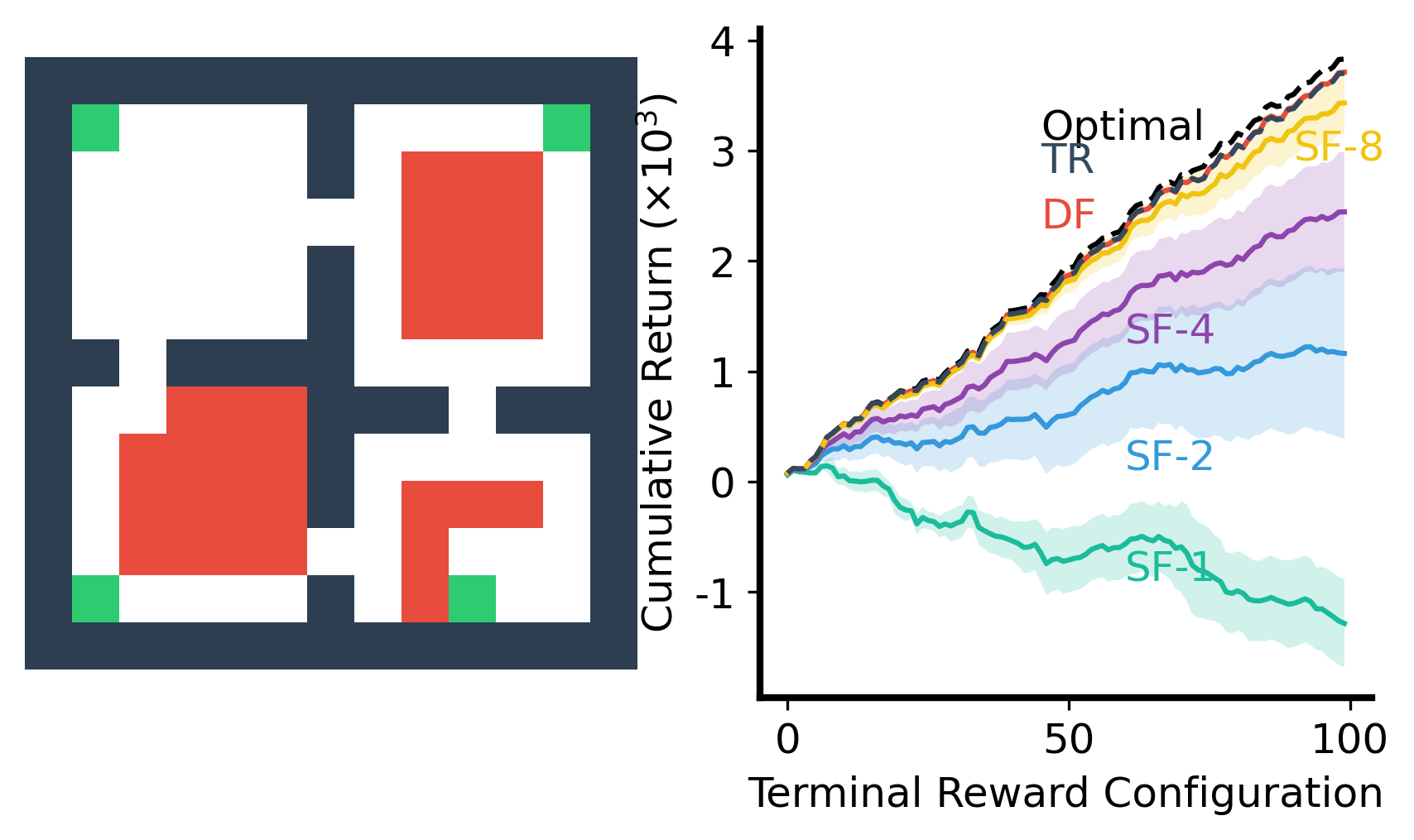}
  \caption{Transfer learning in the multi-goal Four Rooms environment, averaged over 10 seeds. The TR performs similarly to default features, while outperforming all iterations of successor features.}
\end{figure}

The policy-independent DR equivalent, \textit{default features} (DF), views non-terminal rewards as part of the environment dynamics \citep{tse2025reward}. The decoupling is done with respect to terminal rewards only, and features are learned to model the flow of knowledge from terminal states to non-terminal states. This is equivalent to the TR where the choice of features is a one-hot encoding of terminal states. The TR is compared as a tool for transfer with the DF and the SF (under a varying number of base policies). TR performance matches the DF, without the need for an explicit feature learning step, while outperforming all variations of the SF.

\section{Conclusion}
We propose the terminal representation as a novel method for RL state representation. The TR captures knowledge of the underlying information flow in terminating environments, weighted by the desirability of relevant terminal states. We prove convergence of both dynamic programming and sample-based learning algorithms, and show the TR can be used to: 1) zero-shot recover optimal value functions for arbitrary terminal reward specification, 2) perform potential-based reward shaping, 3) achieve intelligent reward-aware exploration during learning, and 4) discover extended macro-actions. Through all this, we note that the TR is more compact than the default representation and successor representation, while not relying on reversible transition dynamics or requiring eigenvector computations---and in fact, the TR inherently encodes information found in the top DR eigenvector. We also provide a result on the equivalence between alternative reward structures in the LMDP setting, allowing our TR theory to translate cleanly between these formulations. Two natural directions for future work are extending to higher-dimensional environments through function approximation, and handling settings without explicitly defined terminal states.

\subsection*{AI use statement}
In this work, we used generative AI tools for helping formulate code to run experiments and visualise their results.
We have not used generative AI tools for writing the paper, finding sources and ideas, or proving theoretical results,
and the rest of the required disclosure tasks are not applicable to this work.
Additionally, we used generative AI tools for proof-reading to find typos and latex errors. We have reviewed all AI-assisted work. We take responsibility for the final content of this work,
including text, claims or artifacts produced with the aid of generative AI.

\subsection*{Reproducibility statement}
Complete proofs of all theorems are given in Appendix~\ref{app:theory}. Experimental settings, hyperparameter search spaces, and selected hyperparameters are given in Appendix~\ref{app:experiments}; our implementation builds on the publicly available code of \citet{tse2025reward} and \citet{machado2020count}. The computation of confidence intervals is described in Appendix~\ref{app:stats}.

\ificlrfinal
\subsubsection*{Acknowledgments}
This work was supported by MLDR (Chist-ERA WAI 2022) PCI2023-145958- 2 (MCIU/AEI/10.13039).
\fi

\bibliography{iclr2027_conference}
\bibliographystyle{iclr2027_conference}


\newpage

\appendix

\section{Theory}\label{app:theory}

In this appendix we prove several of the theorems stated in the main text. For clarity we restate the theorems here.

\renewcommand{\thetheorem}{4.1}
\subsection{Proof of Theorem \ref{dp-conv}}\label{dp-proof}

\begin{theorem}
    Let $ M_0 = D_T$. The update rule
    \begin{align}
        M_{k + 1} = D_T + D_S M_k
    \end{align}
    converges to the TR: $\lim_{k \to \infty} M_{k} = M$.
\end{theorem}

\emph{Proof.} We have $M_{k+1}=\Big[\sum_{i=0}^{k+1}D_S^i\Big]D_T$. Thus, $\lim_{k\to\infty}M_k = \Big[\sum_{i=0}^{\infty}D_S^i\Big]D_T$. The sum can be interpreted as a Neumann series, and we want to show this converges. $||D_S||_{\infty}=\max_{s\in\cS}e^{\frac{\cR(s)}{\lambda}}\sum_{s'\in\cS}\cP^{\mu}(s'\mid s)$. Since $\cR(s)<0$ for all $s\in\cS$, we have $e^{\frac{\cR(s)}{\lambda}}<1$. The probability sum is $\le 1$, since we are considering transitions from some non-terminal state to all other non-terminal states. Therefore, the series converges and we can write $\lim_{k\to\infty}M_k = \Big[\sum_{i=0}^{\infty}D_S^i\Big]D_T=(I_S -D_S)^{-1}D_T=M$. $\blacksquare$

\renewcommand{\thetheorem}{4.2}
\subsection{Proof of Theorem \ref{td-conv}}\label{td-proof}

\begin{theorem}
    Assume the learning rate $\alpha_t$ is such that $\sum_t^{\infty}\alpha_t =\infty$ and $\sum_t^{\infty}\alpha_t^2 <\infty$. The sample-based update rule
    \[
    \widehat M(s_t) \leftarrow (1-\alpha_t) \widehat M(s_t) + \alpha_t e^{\frac{r_t}{\lambda}} \widehat M^+(s_{t+1})
    \]
    will converge to $M$.
\end{theorem}
\emph{Proof:} We can define the operator $F$ acting on rows of $M$ as
\[
(FM)(s) = e^{\frac{\cR(s)}{\lambda}}\sum_{s'\in \cS^+}\cP^{\mu}(s'\mid s)M^+(s').
\]

We also introduce the stochastic noise term $w_t(s_t) = e^{\frac{r_t}{\lambda}} \widehat M^+(s_{t+1}) - (F\widehat M)(s_t)$, which lets us reformat our update in stochastic approximation form:

\[
\widehat M(s_t) \leftarrow (1-\alpha_t)\widehat M(s_t) + \alpha_t ((F\widehat M)(s_t) + w_t(s_t)).
\]

We want to show that the operator $F$ is a contraction. We have

$[FM-FM'](s,\tau) = e^{\frac{\cR(s)}{\lambda}}\sum_{s'\in \cS^+}\cP^{\mu}(s'\mid s)[M^+-M'^+](s',\tau)$, and so applying the triangle inequality followed by bounding each term with the infinity norm we get

\begin{align*}
    |FM-FM'|(s,\tau) &= |e^{\frac{\cR(s)}{\lambda}}\sum_{s'\in \cS^+}\cP^{\mu}(s'\mid s)[M^+-M'^+](s',\tau)| \\
    &\le e^{\frac{\cR(s)}{\lambda}}\sum_{s'\in \cS^+}\cP^{\mu}(s'\mid s)|M^+-M'^+|(s',\tau) \\
    &\le e^{\frac{\cR(s)}{\lambda}}\sum_{s'\in \cS^+}\cP^{\mu}(s'\mid s)\|M^+-M'^+\|_{\infty} \\
    &=e^{\frac{\cR(s)}{\lambda}}\|M^+-M'^+\|_{\infty}\sum_{s'\in \cS^+}\cP^{\mu}(s'\mid s).\\
\end{align*}

The probability sum evaluates to $1$, so this expression becomes $e^{\frac{\cR(s)}{\lambda}}\|M^+-M'^+\|_{\infty}$. For $s\in\cS$, it is trivial that $M^+(s,\tau)=M(s,\tau)$. For $s\in\cT$, we have $M^+(s,\tau)-M'^+(s,\tau)=\mathbb{1}_{s=\tau}-\mathbb{1}_{s=\tau}=0$. Thus we can write

\begin{align*}
    \|FM-FM'\|_{\infty}&\le \max_{s\in\cS}e^{\frac{\cR(s)}{\lambda}}\|M^+-M'^+\|_{\infty} \\
    &= \max_{s\in\cS}e^{\frac{\cR(s)}{\lambda}}\|M-M'\|_{\infty} \\
    &= \gamma\|M-M'\|_{\infty}
\end{align*}

where $\gamma=\max_{s\in\cS}e^{\frac{\cR(s)}{\lambda}}<1$ follows from the fact that $\cR(s)<0$ for all $s\in\cS$. Thus $F$ is a contraction over $\|\cdot\|_{\infty}$. Since $r_t=\cR(s_t)$, the noise $w_t(s_t)$ has zero mean given the history, and $\|w_t\|_\infty\le2\max(1,\|\widehat M\|_\infty)$. With the contraction and the conditions of the theorem, the asynchronous stochastic approximation result of \citet{tsitsiklis1994asynchronous} gives $\widehat M\to M$ with probability 1. $\blacksquare$

\subsection{Proof of Theorem \ref{dr-eigenvec}}\label{eigenvec-proof}

\renewcommand{\thetheorem}{5.1}
\begin{theorem}
Let $\mathbf{\Zeta}$ be the DR over $\mathcal{S}^+$ and $
\mbf \Zeta_S$ its restriction to non-terminal states. Let $\tau^* = \mathrm{argmax}_{\tau\in\mathcal{T}}\mathcal{R}(\tau)$, $\bf e_{\tau^*}$ be a one-hot vector of length $T$ that encodes $\tau^*$, and $M_{\cdot, \tau^*}$ be the corresponding column vector of the TR $M$. Suppose $e^{\cR(\tau^*)/\lambda} > \rho(\mbf \Zeta_S)$, where $\rho(\cdot)$ denotes spectral radius, which holds in particular whenever
$e^{\cR(\tau^*)/\lambda} > \frac{\gamma}{1-\gamma}$ for
$\gamma := \max_{s\in\cS} e^{\cR(s)/\lambda} < 1$. Then the top eigenvector of $\mathbf{\Zeta}$ is
$$\mbf v^* = \begin{bmatrix} \sum_{k=0}^{\infty} e^{-k\frac{\mathcal{R}(\tau^*)}{\lambda}} \mathbf{\Zeta}_{S}^k M_{\cdot, \tau^*} \\ \bf e_{\tau^*} \end{bmatrix}$$
with corresponding eigenvalue $\beta^* = e^{\mathcal{R}(\tau^*)/\lambda}$.
\end{theorem}

\emph{Proof:}
The full DR is given by the $S^+\times S^+$ matrix

\begin{align}
    \mbf \Zeta = \Bigl [R^{-1} - P \Bigr]^{-1}.
\end{align}

To analyse its eigenvectors, we start with the inverted component $\mbf \Zeta^{-1} = \Bigl [R^{-1} - P \Bigr]$. This can be partitioned into submatrices of the form $\mathcal{S}\times \mathcal{S}\to \mathbb{R}, \mathcal{T}\times \mathcal{T} \to \mathbb{R}, \mathcal{S}\times \mathcal{T} \to \mathbb{R}$ and $\mathcal{T}\times \mathcal{S} \to \mathbb{R}$.

The $\mathcal{S}\times \mathcal{S}$ part is just $\Bigl [R_S^{-1} - P_S \Bigr]$.

The $\mathcal{T}\times \mathcal{S}$ part is $\mbf 0_{TS}$, as we are only taking off-diagonal entries from the reward matrix and because there is $0$ probability of transitioning from $\tau\in\mathcal{T}$ to $s\in\mathcal{S}$.

The $\mathcal{S}\times \mathcal{T}$ part will also have off-diagonal 0 entries in the reward matrix, and the transition matrix becomes $P_T$. The full expression is thus $-P_T$.

Since $\mathcal{P}(s\mid\tau) = 0$ for all $s\in\mathcal{S^+}$, the $\mathcal{T}\times \mathcal{T}$ part is just a matrix over terminal rewards, $\diag(y)^{-1}$.

Putting this together, we can say that

$$\mbf \Zeta^{-1} = \begin{bmatrix} (R_S^{-1} - P_S) & -P_T \\ \mbf 0_{TS} & \diag(y)^{-1} \end{bmatrix}.$$

We can invert a matrix of this form \citep{horn2012matrix} to get

\begin{align}
\mbf \Zeta
&= \begin{bmatrix}
(R_S^{-1} - P_S)^{-1} & (R_S^{-1} - P_S)^{-1} P_T \diag(y) \\
\mathbf{0}_{TS} & \diag(y)
\end{bmatrix} \\
&= \begin{bmatrix}
\mbf \Zeta_{S} & M \diag(y)\\
\mathbf{0}_{TS} & \diag(y)
\end{bmatrix}.
\end{align}

Let $\mbf v^{(i)}$ and $\beta_i$ denote respectively the $i^{\text{th}}$ eigenvector of $\mbf \Zeta$ and its corresponding eigenvalue, where the ordering is done by eigenvalue magnitude. Each $\mbf v^{(i)}$ is of length $S^+$, and so we can decompose them into their non-terminal and terminal components: $\mbf v^{(i)} = \left[\mbf v^{(i)}_S \space\mbf v^{(i)}_T\right]$. Then we can write

$\begin{bmatrix} \mathbf{\Zeta}_{S} & M \diag(y) \\ \mathbf{0}_{TS} & \diag(y)\end{bmatrix} \begin{bmatrix} \mbf v^{(i)}_S \\ \mbf v^{(i)}_T \end{bmatrix} = \beta_i \begin{bmatrix} \mbf v^{(i)}_S \\ \mbf v^{(i)}_T \end{bmatrix}$

which leads to
$\mathbf{\Zeta}_{S} \mbf v^{(i)}_S + M \diag(y) \mbf v^{(i)}_T = \beta_i \mbf v^{(i)}_S$ and $\diag(y) \mbf v^{(i)}_T = \beta_i \mbf v^{(i)}_T$, from which we can conclude that $\mbf v^{(i)}_S = \beta_i (\beta_i I - \mathbf{\Zeta}_{S})^{-1} M \mbf v^{(i)}_T$.

So,
\begin{align*}
\mbf v^{(i)}&=\begin{bmatrix} (I - \beta_i^{-1}\mathbf{\Zeta}_{S})^{-1} M \mbf v^{(i)}_T \\ \mbf v^{(i)}_T \end{bmatrix}. \\
\end{align*}

$\diag(y)$ is a diagonal matrix of exponentiated terminal rewards, thus $\diag(y) \mbf v^{(i)}_T = \beta_i \mbf v^{(i)}_T$ tells us $\mbf v^{(i)}_T=e_{\tau}$ (one-hot vector over terminal states) and $\beta_i = e^{\frac{\mathcal{R}(\tau)}{\lambda}}$ for some $\tau\in \mathcal{T}$. If terminal rewards are constant, we have the same $\beta_i$ for all eigenvectors of $\diag(y)$. Eigenvalues arising instead from $\mathbf{\Zeta}_S$ must have $\mbf v^{(i)}_T = \mbf 0$, since under our assumption no eigenvalue of $\mathbf{\Zeta}_S$ is also an eigenvalue of $\diag(y)$, and so the only solution to $\diag(y) \mbf v^{(i)}_T = \beta_i \mbf v^{(i)}_T$ in this case is the zero vector.

Under the assumption $e^{\mathcal{R}(\tau^*)/\lambda} > \rho(\mathbf{\Zeta}_S)$, the largest eigenvalue of the full DR is $\beta^* = e^{\mathcal{R}(\tau^*)/\lambda}$, with $\mbf v^*_T = e_{\tau^*}$. This same assumption guarantees convergence of the Neumann series
\begin{align*}
(I - \beta^{*-1}\mathbf{\Zeta}_{S})^{-1} = \sum_{k=0}^{\infty} \beta^{*-k} \mathbf{\Zeta}_{S}^k
\end{align*}

allowing us to evaluate the top eigenvector of the full DR as

\begin{align*}
\mbf v^*&=\begin{bmatrix} \mbf v^*_S \\ \mbf v^*_T \end{bmatrix} \\
&= \begin{bmatrix} e^{\frac{\mathcal{R}(\tau^*)}{\lambda}} (e^{\frac{\mathcal{R}(\tau^*)}{\lambda}} I- \mathbf{\Zeta}_{S})^{-1}M e_{\tau^*} \\ e_{\tau^*} \end{bmatrix} \\
&= \begin{bmatrix}   (I- e^{-\frac{\mathcal{R}(\tau^*)}{\lambda}}\mathbf{\Zeta}_{S})^{-1}M_{\cdot, \tau^*} \\ e_{\tau^*} \end{bmatrix} \\
&= \begin{bmatrix} \sum_{k=0}^{\infty} e^{-k\frac{\mathcal{R}(\tau^*)}{\lambda}} \mathbf{\Zeta}_{S}^k M_{\cdot, \tau^*} \\ e_{\tau^*}\end{bmatrix}.
\end{align*} $\blacksquare$

\textbf{Further analysis:}

Expanding the series gives

$$\mbf v^* = \begin{bmatrix}  M_{\cdot,\tau^*} + e^{-\frac{\mathcal{R}(\tau^*)}{\lambda}} \mathbf{\Zeta}_{S} M_{\cdot,\tau^*} + e^{-2\frac{\mathcal{R}(\tau^*)}{\lambda}} \mathbf{\Zeta}_{S}^2 M_{\cdot,\tau^*} + \dots  \\ e_{\tau^*}\end{bmatrix}.$$

In this expression, $\mbf v^*_S$ evaluates to a vector of length $S$, where each entry gives us a measure relating the current non-terminal state to the highest reward terminal state. $\mbf v^*_T$ identifies the highest reward terminal state among the index of all terminal states. If we have $\cR(\tau^*)>0$, then the exponentiated reward components in $\mbf v^*_S$ are decaying with index $k$, and so we can interpret this as a discount factor. For applications like reward shaping, the entries of $\mbf v^*$ at particular states are used. Clearly for $\tau\in\mathcal{T}, \mbf v^* (\tau) = \mathbb{1}_{\tau=\tau^*}$. For $s\in\mathcal{S}$, we have $$\mathbf{v}^*(s) = \sum_{k=0}^{\infty} e^{-k\frac{\mathcal{R}(\tau^*)}{\lambda}} \sum_{s' \in \mathcal{S}} \mathbf{\Zeta}_{S}^k(s, s') M(s', \tau^*).$$

Recall the definitions $\mbf \Zeta_{S} = (R_S^{-1} - P_S)^{-1}, M = (R_S^{-1} - P_S)^{-1}P_T$. Then,
\begin{align*}
    \sum_{k=0}^{\infty} e^{-k\frac{\mathcal{R}(\tau^*)}{\lambda}} \mathbf{\Zeta}_{S}^k M_{\cdot, \tau^*} &= \sum_{k=0}^{\infty} e^{-k\frac{\mathcal{R}(\tau^*)}{\lambda}} (R_S^{-1} - P_S)^{-(k+1)}P_{\tau^*} \\
    &= \bigg(\sum_{k=0}^{\infty} e^{-k\frac{\mathcal{R}(\tau^*)}{\lambda}} (R_S^{-1} - P_S)^{-(k+1)}\bigg)P_{\tau^*}.
\end{align*}

This can be interpreted as a discounted reward-aware transition term over trajectories of different lengths, weighted by transitions to the most desirable terminal state. 

\textbf{Sufficient condition.} Let $\gamma := \max_{s\in\cS} e^{\cR(s)/\lambda} < 1$. Since $\|D_S\|_\infty\le\gamma$, the Neumann series gives $\|(I_S-D_S)^{-1}\|_\infty\le\frac{1}{1-\gamma}$. With $\mbf\Zeta_S=(I_S-D_S)^{-1}R_S$ and $\|R_S\|_\infty=\gamma$,
\[
\rho(\mbf\Zeta_S)\le\|\mbf\Zeta_S\|_\infty\le\frac{\gamma}{1-\gamma},
\]
so $e^{\cR(\tau^*)/\lambda}>\frac{\gamma}{1-\gamma}$ implies $e^{\cR(\tau^*)/\lambda}>\rho(\mbf\Zeta_S)$.




\subsection{Continuous state spaces}\label{cont-proof}
Let $\cS$ be a measurable space, $\cT$ finite, and $B(\cS)$ the Banach space of bounded measurable functions under the norm $\|\cdot\|_\infty$. Assume $\cR$ is measurable with $\sup_{s\in\cS}\cR(s)<0$, so that $\gamma:=\sup_{s\in\cS}e^{\cR(s)/\lambda}<1$.

\emph{Existence:} For $z\in B(\cS)$, $|(D_Sz)(s)|\le e^{\cR(s)/\lambda}\,\cP^\mu(\cS\mid s)\,\|z\|_\infty\le\gamma\|z\|_\infty$. Hence $\|D_S\|\le\gamma<1$, $(I_S-D_S)^{-1}=\sum_{i\ge0}D_S^i$ is bounded, and $M=\sum_{i\ge0}D_S^iD_T$ satisfies $M=D_T+D_SM$.

\emph{Pointwise definition:} All quantities are non-negative, so
\[
[D_S^iD_T](s,\tau)=\int e^{\frac{\cR(s)}{\lambda}}\cP^\mu(ds_1\mid s)\cdots e^{\frac{\cR(s_{i-1})}{\lambda}}\cP^\mu(ds_i\mid s_{i-1})\,e^{\frac{\cR(s_i)}{\lambda}}\cP^\mu(\tau\mid s_i),
\]
and $M(s,\tau)=\sum_{i\ge0}[D_S^iD_T](s,\tau)$. Since $\sum_\tau D_T(s,\tau)\le\gamma$ and $\int e^{\cR(s)/\lambda}\cP^\mu(ds'\mid s)\le\gamma$, induction gives $\sum_\tau[D_S^iD_T](s,\tau)\le\gamma^{i+1}$, hence $\sum_\tau M(s,\tau)\le\frac{\gamma}{1-\gamma}$ for every $s$.

\emph{Value Recovery:} Let $L(s)=\sum_\tau M(s,\tau)y(\tau)$, so $\|L\|_\infty\le\frac{\gamma}{1-\gamma}\|y\|_\infty$. Separating the $i=0$ term gives $L=D_Ty+D_SL$. Since $I_S-D_S$ is invertible, $L=(I_S-D_S)^{-1}D_Ty=z$. $\blacksquare$

\subsection{Proof of Theorem \ref{equiv}}\label{equiv-proof}

\renewcommand{\thetheorem}{6.1}
\begin{theorem}
   Consider an LMDP of the form $\widetilde{\cL}=\left( \cS,\cT,\cA,\cP,\widetilde{\cR},\widetilde{\mu}\right)$, with \textit{transition} dependent reward $\widetilde\cR:\cS\times\cS^+\to\mathbb{R}$. We can always transform $\widetilde \cL$ to an equivalent LMDP
\[
  \cL=\left( \cS,\cT,\cA,\cP,\cR,\mu\right)
\]
with state-dependent rewards $\cR:\cS^+\to\mathbb{R}$, such that $\cL$ and $\widetilde{\cL}$ have the same optimal value function.
\end{theorem}

\emph{Proof:}
For each $(s,s')\in \cS\times \cS^+$, define
\[
  x(s,s') = \log \cP^{\widetilde \mu}(s'\mid s) + \frac{\widetilde\cR(s,s')}{\lambda}.
\]
Note that the Bellman optimality equation can be written as
\[
  z(s) = \sum_{s'\in \cS^+} e^{x(s,s')}z(s'), \quad \forall s\in \cS.
\]

The idea is to choose $\mu$ and $\cR$ such that
\[
  \log\cP^{\mu}(s'\mid s)+\frac{\cR(s)}{\lambda}=x(s,s')
\]
for each $(s,s')\in \cS\times \cS^+$. To compute $\cR(s)$ we can write
\[
  \sum_{s'} e^{x(s,s')}
  = e^{\cR(s)/\lambda}\sum_{s'}\cP^{\mu}(s'\mid s)
  = e^{\cR(s)/\lambda}
  \iff
  \cR(s)=\lambda\log\sum_{s'}e^{x(s,s')}.
\]
In turn, this allows us to compute $\cP^{\mu}(s'\mid s)$ as
\begin{align*}
  &\log\cP^{\mu}(s'\mid s)
  = x(s,s') - \log\sum_{s''}e^{x(s,s'')}\\
  \iff \quad
   &\cP^{\mu}(s'\mid s)
  = \frac{e^{x(s,s')}}{\sum_{s''}e^{x(s,s'')}} = \frac{e^{\widetilde\cR(s,s')/\lambda} \cP^{\widetilde\mu}(s'\mid s)} {\sum_{s''} e^{\widetilde\cR(s,s'')/\lambda} \cP^{\widetilde\mu}(s''\mid s)}.
\end{align*}
Here we have to assume that there exists a policy $\mu$ that realizes the transition kernel $\cP^\mu$.

Under the assumption that $\tilde\cR(s,s')<0$ for each $(s,s')\in \cS\times \cS^+$, we obtain
\[
  \cR(s)
  = \lambda\log\sum_{s'}e^{x(s,s')}
  = \lambda\log\sum_{s'}\cP^{\widetilde\mu}(s'\mid s)e^{\widetilde\cR(s,s')/\lambda}
  < \lambda\log\sum_{s'}\cP^{\widetilde\mu}(s'\mid s)
  = 0.
\]
Hence $\cL$ satisfies $\cR(s)<0$ for each $s\in \cS$. $\blacksquare$

\subsection{Proof of Theorem \ref{sa-tr}}\label{sa-proof}
\renewcommand{\thetheorem}{6.2}
\begin{theorem}
    Consider an LMDP of the form $\bar{\cL}=(\cS,\cT,\cA,\cP,\bar{\cR},\mu)$, where the reward is given by $\bar \cR :\cS^+\times\cA\to\mathbb{R}$. Let $q$ be the optimal action-value function satisfying Bellman equations of the form $q=\bar{\mathcal R}+\lambda\log\mathbb{E}_{s'}\big[\sum_{a'}\mu(a'|s')e^{q(s',a')/\lambda}\big]$. Now we define the state-action exponentiated values as $\bar z(s,a)=e^\frac{q(s,a)}{\lambda}$ for $s\in\cS$ and $\bar z(\tau,a)=e^\frac{\bar \cR(\tau,a)}{\lambda}$ for $\tau\in\cT$. Let $\bar z \in \mathbb{R}^{S\times A}$ be the vector of exponentiated non-terminal action-values and $\bar y \in \mathbb{R}^{T\times A}$ its terminal equivalent. Then, we can compute the state-action TR $\bar M$ under this formulation and use it to recover action-values through the relation $\bar z = \bar M\bar y$.
\end{theorem}

\emph{Proof:} The induced distribution $\bar \cP^{\mu}:\cS^+\times\cA\to\Delta(\cS^+\times\cA)$ is computed as $\bar \cP^{\mu}(s',a'\mid s,a) = \cP(s'\mid s, a) \cdot \mu(a'\mid s')$, with $\bar \cP^{\mu}(s',a'\mid \tau,a) = 0$ for all $\tau\in\cT$.

$\bar R_{SA} \in \mathbb{R}^{SA \times SA}$ is a diagonal exponentiated reward matrix for non-terminal, while $\bar P_{SA} \in \mathbb{R}^{SA \times SA}$ and $\bar P_{TA} \in \mathbb{R}^{SA \times TA}$ encode transitions.

We can partition the Bellman equations as in the state-dependent case: For all $s\in\cS$ and $a\in\cA$,
\begin{align*}
\bar z(s,a) &= e^\frac{\bar \cR(s,a)}{\lambda}\sum_{s',a'\in \cS^+\times\cA}\bar\cP^{\mu}(s',a'\mid s,a) \bar z(s',a') \\
&= e^\frac{\bar\cR(s,a)}{\lambda}\sum_{s',a'\in \cS\times\cA}\bar\cP^{\mu}(s',a'\mid s,a) \bar z(s',a') + e^\frac{\bar\cR(s,a)}{\lambda}\sum_{\tau,a'\in \cT\times\cA}\bar\cP^{\mu}(\tau,a'\mid s,a) e^\frac{\bar\cR(\tau,a')}{\lambda}.
\end{align*}

This leads to
\[
  \bar z = \bar R_{SA} \bar P_{SA} \bar z + \bar R_{SA} \bar P_{TA} \bar y
  \iff
  \bar z = (\bar R_{SA}^{-1}-\bar P_{SA})^{-1}\bar P_{TA}\bar y =(I_{SA}-\bar R_{SA} \bar P_{SA})^{-1}\bar R_{SA}\bar P_{TA}\bar y.
\]

We introduce equivalent compact matrices to the state-dependent case: $\bar D_{SA}=\bar R_{SA} \bar P_{SA}$ and $\bar D_{TA}=\bar R_{SA} \bar P_{TA}$.

Now,
\begin{align*}
     \bar M &= (\bar R_{SA}^{-1}-\bar P_{SA})^{-1}\bar P_{TA}\\
    &= (I_{SA}-\bar R_{SA} \bar P_{SA})^{-1}\bar R_{SA}\bar P_{TA} \\
    &= (I_{SA}-\bar D_{SA})^{-1} \bar D_{TA}
\end{align*}

is the \textit{state-action} TR, and we can compute $\bar z = \bar M\bar y$. $\blacksquare$

\subsection{Proof of Theorem \ref{tr-comp}}\label{comp-proof}

\renewcommand{\thetheorem}{6.3}
\begin{theorem}
Let $z_i$ be the value of the base LMDP $\cL_i$ with terminal values $z_i(\tau)=\mathbb{1}_{\tau=\tau_i}$ (i.e.~$\cR_i(\tau_i)=0$, $\cR_i(\tau)=-\infty$ otherwise). Then $M^+(\cdot,\tau_i)=z_i$ for $i=1,\dots,T$, and hence the TR can be used to perform compositionality via the relation $z(s)=\sum_{i=1}^T z(\tau_i)M^+(s,\tau_i)$.
\end{theorem}

\emph{Proof:}

The TR matrix $M$ can be interpreted as a function $M:\mathcal{S}\times\mathcal{T}\to\mathbb{R}$. Then, for $s\in \mathcal{S}$ and $\tau_i\in\mathcal{T}$, our recursion $M={D}_T + {D}_S M$ can be written element-wise as

\begin{equation*}
    M(s,\tau_i) = e^{\frac{\mathcal{R}(s)}{\lambda}}(\mathcal{P}^{\mu}(\tau_i\mid s) + \sum_{s'\in\mathcal{S}} \mathcal{P}^{\mu}(s'\mid s) M(s',\tau_i)).
\end{equation*}

This is not quite the Bellman equation in $M$, because of the separation of $\mathcal{P}^{\mu}$ into $P_S$ and $P_T$.

Returning to the previous recursion and using the extended domain $M^+$ defined above, we can write:

$\begin{aligned}
M(s,\tau_i) &= e^{\frac{\mathcal{R}(s)}{\lambda}}(\mathcal{P}^{\mu}(\tau_i\mid s) + \sum_{s'\in\mathcal{S}} \mathcal{P}^{\mu}(s'\mid s) M(s',\tau_i)) \\
&= e^{\frac{\mathcal{R}(s)}{\lambda}}(\sum_{\tau'\in\mathcal{T}}\cP^{\mu}(\tau'\mid s)M^+(\tau',\tau_i) + \sum_{s'\in\mathcal{S}} \cP^{\mu}(s'\mid s) M^+(s',\tau_i)) \\
&= e^{\frac{\mathcal{R}(s)}{\lambda}} \sum_{s'\in\mathcal{S}^+} \cP^{\mu}(s'\mid s) M^+(s',\tau_i).
\end{aligned}$

This is not a valid recursion, since the LHS uses $M$ while the RHS uses $M^+$. For $s\in\mathcal{S}$ we have $M(s,\tau_i) = M^+(s,\tau_i)$, and so it is trivial to replace the LHS. This gives us the following expression for non-terminal states:

\begin{equation*}
    M^+(s,\tau_i)=e^{\frac{\mathcal{R}(s)}{\lambda}} \sum_{s'\in\mathcal{S}^+} \cP^{\mu}(s'\mid s) M^+(s',\tau_i).
\end{equation*}

We now examine the boundary conditions, i.e. states $\tau'\in\mathcal{T}$. By definition, we have $M^+(\tau',\tau_i)=\mathbb{1}_{\tau'=\tau_i}$.

$z_i(s)$, the exponentiated value function for base LMDP $i$, must satisfy the exponentiated Bellman equation for non-terminal states:
$$z_i(s) = e^{\frac{\mathcal{R}_i(s)}{\lambda}} \sum_{s'\in\mathcal{S}^+} \cP^{\mu}(s'\mid s) z_i(s').$$

Since the base LMDPs share the same non-terminal rewards as the overarching LMDP, $\mathcal{R}_i(s) = \mathcal{R}(s)$ for all $s \in \mathcal{S}$. Therefore, $z_i(s)$ obeys the same recursion as $M^+(s, \tau_i)$.

Under the assumed reward structure, we have that the terminal values $z_i(\tau)=e^{\cR_i(\tau)/\lambda}=\mathbb{1}_{\tau=\tau_i}$

Because $M^+(s, \tau_i)$ and $z_i(s)$ satisfy identical recursive equations for non-terminal states and share identical boundary conditions at terminal states, and since the Bellman equations have a unique solution, we can conclude that $M^+(s, \tau_i) = z_i(s)$ for all $s \in \mathcal{S}^+$ and $i=1\dots T$.

We have that $z(s) = \sum_{i=1}^T z(\tau_i)z_i(s)$. By direct substitution, we arrive at the final result:
$$z(s) = \sum_{i=1}^T z(\tau_i)M^+(s,\tau_i).$$ $\blacksquare$

\section{Additional Experimental Details}\label{app:experiments}
Experiments are adapted from the work of \citet{tse2025reward}. Details regarding the relevant settings, as well as hyperparameter configurations relating to the DR and SR, can be found in Appendix C of their paper. The code they used is publicly available at \url{https://github.com/httse9/Reward-Aware-Proto-Representations}, and we use this as a basis to implement our TR experiments. An independent repository for this paper will be made available at a later stage. \citet{tse2025reward} themselves adapt their Count-Based Exploration experiments and implementation from \citet{machado2020count}, with code for this work found at \url{https://github.com/mcmachado/count_based_exploration_sr}.

Experimental settings and hyperparameters are thus used as reported for the SR, the DR, and any 'neutral' baselines---e.g. the random walk in option discovery, SARSA without exploration bonus in count based exploration, learning with no shaped reward in reward shaping. We thus focus on details directly relevant to the TR experiments within this framework. Number of seeds used to generate confidence intervals are mentioned in the plot captions of the main text. Where this is fewer than what is reported by \citet{tse2025reward} (or \citet{machado2020count}), all experiments are run independently from scratch over the same seeds and numbers of seeds. Note that in general, though we use reported hyperparameters, experiments are run from scratch and we do not use stored results. Across the experiments, we do not perform comparisons under unequal conditions. It was necessary to modify the setting of one experiment in order to accommodate the TR---option discovery. We discuss this in the following subsection.

\subsection{Option discovery}
As acknowledged in the main text, the core limitation of the TR is its dependence on explicitly defined terminal states. \citet{tse2025reward} conduct their option discovery experiments in environments without terminal states, which naturally prevents usage of the TR. Our option
discovery experiment is conducted in a Four Rooms environment with four goal states, each of which inducing termination when reached. The environment layout and goal placement is identical to the multi-goal Four Rooms used by \citet{tse2025reward}, but the original version treats these goals as non-terminal. We make another variation to this environment in that the goal rewards are differently weighted---+4 in the top left, +3 in the bottom left, +2 in the bottom right, and +1 in the top right. The original version uses a reward of 0 for all goals. This change is to induce meaningful difference across the four TR variants.

The hyperparameter search space for option discovery follows Table 2 of \citet{tse2025reward}, and is reproduced in Table~\ref{tab:od_search}. Best hyperparameters are selected by maximising a composite score of normalised state visitation percentage and normalised average reward,
evaluated over the first 1500 steps across a minimum of 5 independent seeds. The Q-learning hyperparameter search sweeps over initialisation values $Q_\text{init} \in \{-10, 0\}$, exploration rates $\varepsilon \in \{0.15, 0.2\}$, and step sizes $\in \{0.01, 0.3, 1.0\}$. The same combination proved optimal across all methods. Best hyperparameters found for each method are reported in Tables~\ref{tab:od_best}
and~\ref{tab:od_qbest}.

\begin{table}[H]
\centering
\caption{Hyperparameter search space for option discovery.}
\label{tab:od_search}
\begin{tabular}{llr}
\toprule
\textbf{Name} & \textbf{Description} & \textbf{Values} \\
\midrule
$p_\text{option}$   & Probability of selecting an option                      & $[0.01, 0.05, 0.1]$ \\
$N_\text{learn}$    & Number of sweeps through dataset to learn the representation  & $[1, 10, 100]$ \\
$\alpha$            & Step size for learning the representation               & $[0.01, 0.03, 0.1]$ \\
$N_\text{option}$   & Number of latest options to keep                        & $[1, 8, 1000]$ \\
\bottomrule
\end{tabular}
\end{table}

\begin{table}[H]
\centering
\caption{Best hyperparameters for all methods on \texttt{fourrooms\_multigoal}.}
\label{tab:od_best}
\begin{tabular}{lcccc}
\toprule
\textbf{Method} & $p_\text{option}$ & $N_\text{learn}$ & $\alpha$ & $N_\text{option}$ \\
\midrule
CEO (SR)        & $0.1$  & $1$  & $0.01$ & $8$ \\
RACE (DR)       & $0.05$ & $1$  & $0.03$ & $1$ \\
TR-Cyclical     & $0.05$ & $1$  & $0.01$ & $8$ \\
TR-Random       & $0.1$  & $10$ & $0.03$ & $1$ \\
TR-Highest      & $0.1$  & $1$  & $0.01$ & $1$ \\
TR-All          & $0.1$  & $1$  & $0.01$ & $8$ \\
\bottomrule
\end{tabular}
\end{table}

\begin{table}[H]
\centering
\caption{Best Q-learning hyperparameters on \texttt{fourrooms\_multigoal}.}
\label{tab:od_qbest}
\begin{tabular}{lccc}
\toprule
\textbf{Method} & $Q_\text{init}$ & $\varepsilon$ & Step size \\
\midrule
Pure Q          & $-10$ & $0.2$ & $0.3$ \\
CEO+Q (SR)      & $-10$ & $0.2$ & $0.3$ \\
RACE+Q (DR)     & $-10$ & $0.2$ & $0.3$ \\
TR-Cyclical+Q   & $-10$ & $0.2$ & $0.3$ \\
TR-Random+Q     & $-10$ & $0.2$ & $0.3$ \\
TR-Highest+Q    & $-10$ & $0.2$ & $0.3$ \\
TR-All+Q        & $-10$ & $0.2$ & $0.3$ \\
\bottomrule
\end{tabular}
\end{table}

\subsection{Count based exploration}

Hyperparameters for Sarsa, Sarsa+SR, and Sarsa+DR are used as reported by \citet{machado2020count} and \citet{tse2025reward} respectively, and are reproduced in Tables~\ref{tab:cbe_sarsa}--\ref{tab:cbe_dr}. For Sarsa+TR, we perform a random search over 80 configurations sampled from the space given in Table~\ref{tab:cbe_search}. Best hyperparameters are reported in Table~\ref{tab:cbe_tr}.

\begin{table}[H]
\centering
\caption{Hyperparameters for Sarsa, reproduced from \citet{machado2020count}.}
\label{tab:cbe_sarsa}
\begin{tabular}{lccc}
\toprule
\textbf{Environment} & $\alpha$ & $\varepsilon$ & $\gamma$ \\
\midrule
\textsc{RiverSwim} & $0.37$ & $0.12$ & $0.95$ \\
\textsc{SixArms}   & $0.43$ & $0.01$ & $0.95$ \\
\bottomrule
\end{tabular}
\end{table}

\begin{table}[H]
\centering
\caption{Hyperparameters for Sarsa+SR, reproduced from \citet{machado2020count}.}
\label{tab:cbe_sr}
\begin{tabular}{lcccc}
\toprule
\textbf{Environment} & $\alpha$ & $\eta$ & $\beta$ & $\gamma_\phi$ \\
\midrule
\textsc{RiverSwim} & $0.1$ & $0.5$  & $10000$ & $0.5$ \\
\textsc{SixArms}   & $0.5$ & $0.25$ & $10000$ & $0.5$ \\
\bottomrule
\end{tabular}
\end{table}

\begin{table}[H]
\centering
\caption{Hyperparameters for Sarsa+DR, reproduced from \citet{tse2025reward}.}
\label{tab:cbe_dr}
\begin{tabular}{lccccc}
\toprule
\textbf{Environment} & $\alpha$ & $\eta$ & $\beta$ & $\lambda$ & transform \\
\midrule
\textsc{RiverSwim} & $0.5$ & $0.25$ & $100$ & $1.0$ & $\log(\|\mathbf{x}\|_2)$ \\
\textsc{SixArms}   & $0.5$ & $0.01$ & $0.1$ & $1.5$ & $\log(\|\mathbf{x}\|_2)$ \\
\bottomrule
\end{tabular}
\end{table}

\begin{table}[H]
\centering
\caption{Hyperparameter search space for Sarsa+TR.}
\label{tab:cbe_search}
\begin{tabular}{llr}
\toprule
\textbf{Name} & \textbf{Description} & \textbf{Values} \\
\midrule
$\alpha$      & Step size for Q-value updates        & $[0.01, 0.1, 0.25, 0.5]$ \\
$\eta$        & Step size for TR updates             & $[0.01, 0.1, 0.25, 0.5]$ \\
$\beta$       & Intrinsic reward scaling factor      & $[0.1, 1, 10, 100, 1000, 10000]$ \\
$\lambda$     & Exponential weighting of cumulant    & $[1.0, 1.5, 2.0]$ \\
$\gamma$      & TR discount factor                   & $[0.95, 1.0]$ \\
transform     & Transformation applied to TR norm    & $[\log\|\cdot\|_1,\ \log\|\cdot\|_2]$ \\
\bottomrule
\end{tabular}
\end{table}

\begin{table}[H]
\centering
\caption{Best hyperparameters for Sarsa+TR.}
\label{tab:cbe_tr}
\begin{tabular}{lcccccc}
\toprule
\textbf{Environment} & $\alpha$ & $\eta$ & $\beta$ & $\lambda$ & $\gamma$ & transform \\
\midrule
\textsc{RiverSwim} & $0.25$ & $0.01$ & $10$ & $2.0$ & $0.95$ & $\log(\|\mathbf{x}\|_2)$ \\
\textsc{SixArms}   & $0.01$ & $0.5$  & $10$ & $1.5$ & $0.95$ & $\log(\|\mathbf{x}\|_2)$ \\
\bottomrule
\end{tabular}
\end{table}

\section{Statistical significance}\label{app:stats}
We report 95\% confidence intervals throughout. These capture randomness across independent runs, where the primary source of randomness is random exploration via $\varepsilon$-greedy
policies. In option discovery and transfer experiments, randomness also arises from the stochastic environment dynamics across seeds. Following \citet{tse2025reward}, we compute 95\% confidence intervals as $\pm 1.96\frac{\hat{\sigma}}{\sqrt{N}}$, where $\hat{\sigma}$ is the sample standard deviation. For experiments conducted over $N \geq 50$ seeds this approximation is well-justified; for smaller $N$ (as in option discovery, where $N = 10$) it should be interpreted as approximate.

\section{Compute resources}

All experiments were conducted on a standard laptop CPU. Run times are modest, but can become more prohibitive with higher number of seeds and parameter configurations. This is especially the case for option discovery experiments.

\end{document}